\documentclass[12pt]{article}

\usepackage{libertine}
\usepackage{amsmath, amsthm, amssymb}

\usepackage{algorithm}
\usepackage{algpseudocode}

\usepackage{hyperref}
\hypersetup{
    colorlinks=true,
    linkcolor=blue,
    filecolor=magenta,      
    urlcolor=teal,
    citecolor=teal,
}

\usepackage{cleveref}
\usepackage{url}

\usepackage{graphicx}
\DeclareGraphicsExtensions{.pdf,.png,.jpg}

\usepackage[table]{xcolor}

\usepackage[margin=4cm]{geometry}

\usepackage{caption}
\usepackage[T1]{fontenc}
\usepackage[all]{xy}
\usepackage[inline]{enumitem}

\usepackage{tikz}
\usepackage{tikz-cd}
\usepackage{pgfplots}
\pgfplotsset{compat=newest}

\usepackage{siunitx}

\makeatletter
\let\@fnsymbol\@arabic
\makeatother

\newtheorem{theorem}{Theorem}

\newtheorem{remark}{Remark}

\newtheorem{definition}{Definition}
\newtheorem{proposition}{Proposition}

\definecolor{sienna1}{rgb}{0.912428,0.879103,0.810742}
\definecolor{sienna2}{rgb}{0.923632,0.792035,0.603304}
\definecolor{sienna3}{rgb}{0.915381,0.591867,0.320798}
\definecolor{sienna4}{rgb}{0.794293,0.358873,0.0979057}
\definecolor{sienna5}{rgb}{0.466987,0.173327,0.0693065}

\DeclareMathOperator{\img}{Im}
\DeclareMathOperator{\spn}{span}

\newcommand{\R}{{\mathbb R}}
\newcommand{\N}{{\mathbb N}}
\newcommand{\f}{{\tilde f}}

\newcommand{\h}{{\tilde h}}
\newcommand{\x}{{\tilde x}}
\renewcommand{\L}{{\tilde L}}
\newcommand{\V}{{\tilde V}}

\newcommand{\id}{{\textnormal{id}}}

\newcommand{\range}[2]{{\left\{{#1}, \dots,{#2}\right\}}}
\newcommand{\pluseq}{\mathrel{+}=}

\crefname{diagram}{diag.}{diags.}
\Crefname{diagram}{Diagram}{Diagrams}
\creflabelformat{diagram}{#2(#1)#3}

\tikzstyle{emptyvertex} = [inner sep=0cm,node distance=2.4cm]
\tikzstyle{vertex} = [fill,shape=circle,minimum size=0.8cm,inner sep=0cm,color=white]
\tikzstyle{edge} = [fill,line cap=round,line join=round,line width=1.15cm]
\tikzstyle{elabel} = [fill,shape=circle,node distance=1cm]

\pgfdeclarelayer{background}
\pgfsetlayers{background,main}
\title{
    Machines of finite depth: towards a formalization of neural networks
    }
\author{
    Pietro Vertechi
    \thanks{Correspondence at \url{pietro.vertechi@protonmail.com}}
    \and Mattia G. Bergomi
    \thanks{Correspondence at \url{mattiagbergomi@gmail.com}}
}
\date{}

\begin{document}
\maketitle

\begin{abstract}
    We provide a unifying framework where artificial neural networks and their architectures can be formally described as particular cases of a general mathematical construction---{\em machines of finite depth}. Unlike neural networks, machines have a precise definition, from which several properties follow naturally. Machines of finite depth are modular (they can be combined), efficiently computable and differentiable. The backward pass of a machine is again a machine and can be computed without overhead using the same procedure as the forward pass. We prove this statement theoretically and practically, via a unified implementation that generalizes several classical architectures---dense, convolutional, and recurrent neural networks with a rich shortcut structure---and their respective backpropagation rules.
\end{abstract}

\section{Introduction}

The notion of {\em artificial neural network} has become more and more ill-defined over time. Unlike the initial definitions~\cite{rumelhart1986learning}, which could be easily formalized as directed graphs, modern neural networks can have the most diverse structures and do not obey a precise mathematical definition.

Defining a deep neural network is a practical question, which must be addressed by all deep learning software libraries. Broadly, two solutions have been proposed. The simplest approach defines a deep neural network as a stack of pre-built layers. The user can select among a large variety of pre-existing layers and define in what order to compose them. This approach simplifies the end-user's mental load: in principle, it becomes possible for the user to configure the model via a simplified {\em domain-specific language}. It also leads to computationally efficient models, as the rigid structure of the program makes its optimization easier for the library's software developers. Unfortunately, this approach can quickly become limiting and prevent users from exploring more innovative architectures~\cite{barham2019machine}.
At the opposite end of the spectrum, a radically different approach, {\em differentiable programming}~\cite{wang2018backpropagation}, posits that every code is a model, provided that it can be differentiated by an automatic differentiation engine~\cite{frostig2018compiling,innes2019differentiable,NEURIPS2020_9332c513,paszke2017automatic,paszke2021getting,saeta2021swift,van2018automatic}. This is certainly a promising direction, which has led to a number of technological advances, ranging from differentiable ray-tracers to neural-network-based solvers for partial differential equations~\cite{innes2019differentiable,rackauckas2020universal,zubov2021neuralpde}. Unfortunately, this approach has several drawbacks, some practical and some theoretical. On the practical side, it becomes difficult to optimize the runtime of the forward and backward pass of an automatically-differentiated, complex, unstructured code. On the other hand, a mathematical formalization would allow for an efficient, {\em unified} implementation. From a more theoretical perspective, the {\em space of models} becomes somewhat ill-defined, as it is now the space of {\em all differentiable codes}---not a structured mathematical space. This concern is not exclusively theoretical. A well-behaved smooth space of neural networks would be invaluable for automated differentiable {\em architecture search}~\cite{liu2018darts}, where the optimal network structure for a given problem is found automatically. Furthermore, a well-defined notion of neural network would also foster collaboration, as it would greatly simplify {\em sharing} models as precise mathematical quantities rather than differentiable code written in a particular framework.

\paragraph{Aim.} Our ambition is to establish a unified framework for deep learning, in which deep feedforward and recurrent neural networks, with or without shortcut connections, are defined in terms of a {\em unique} layer, which we will refer to as a {\em parametric machine}.
This approach allows for extremely simplified flows for designing neural architectures, where a small set of hyperparameters determines the whole architecture.
By virtue of their precise mathematical definition, parametric machines will be language-agnostic and independent from automatic differentiation engines. Computational efficiency, in particular in terms of efficient {\em gradient} computation, is granted by the mathematical framework.

\paragraph{Contributions.}
The theoretical framework of parametric machines unifies seemingly disparate architectures, designed for structured or unstructured data, with or without recurrent or shortcut connections. We provide theorems ensuring that
\begin{enumerate*}
    \item under the assumption of finite depth, the output of a machine can be computed efficiently;
    \item complex architectures with shortcuts can be built by adding together machines of depth one, thus generalizing neural networks at any level of granularity (neuron, layer, or entire network);
    \item backpropagating from output to input space is again a machine computation and has a computational cost comparable to the forward pass.
\end{enumerate*}
In addition to the theoretical framework, we implement the input-output computations of parametric machines, as well as their derivatives, in the Julia programming language~\cite{bezansonJuliaFreshApproach2017} (both on CPU and on GPU). Each algorithm can be used both as standalone or layer of a classical neural network architecture.

\paragraph{Structure.}
\Cref{sec:machines} introduces the abstract notion of {\em machine}, as well as its theoretical properties. In \cref{sec:resolvent}, we define the {\em machine equation} and the corresponding {\em resolvent}. There, we establish the link with deep neural networks and backpropagation, seen as machines on a global normed vector space. \Cref{sec:depth} discusses under what conditions the machine equation can be solved efficiently, whereas in \cref{sec:composability} we discuss how to combine machines under suitable independence assumptions. The theoretical framework is completed in \cref{sec:optimization}, where we introduce the notion of parametric machine and discuss e

xplicitly how to differentiate its output with respect to the input and to the parameters. \Cref{sec:implementation} is devoted to practical applications. There, we discuss in detail an implementation of machines that extends classical dense, convolutional, and recurrent networks with a rich shortcut structure.

\section{Machines}
\label{sec:machines}

We start by setting the mathematical foundations for the study of {\em machines}. In order to retain two key notions that are pervasive in deep learning---linearity and differentiability---we choose to work with normed vector spaces and Fr\'{e}chet derivatives. We then proceed to build network-like architectures starting from continuously differentiable maps of normed vector spaces. We refer the reader to \cref{sec:normedvectorspaces} for relevant definitions and facts concerning differentiability in the sense of Fr\'{e}chet.

The key intuition is that a neural network can be considered as an endofunction $f:X\rightarrow X$ on a space of global functions $X$ (defined on all neurons on all layers). We will show that this viewpoint allows us to recover classical neural networks with arbitrarily complex shortcut connections. In particular, the forward pass of a neural network corresponds to computing the inverse of the mapping $\id - f$. We explore under what conditions on $f$, the mapping $\id - f$ is invertible, and provide practical strategies for computing it and its derivative.

This meshes well with the recent trend of {\em deep equilibrium models}~\cite{bai2019deep}. There, the output of the network is defined implicitly, as the solution of a fixed-point problem. Under some assumptions, such problems have a unique solution that can be found efficiently~\cite{winston2020monotone}. Furthermore, the implicit function theorem can be used to compute the derivative of the output with respect to the input and parameters~\cite{gurumurthy2021joint}. While the overarching formalism is similar, here we choose a different set of fixed-point problems, based on an {\em algebraic} condition which generalizes classical feedforward architectures and does not compromise on computational efficiency.
We explored in~\cite{vertechi2020parametric} a first version of the framework. Here, we develop a much more streamlined approach, which does not rely explicitly on category theory. Instead, we ground the framework in functional analysis. This perspective allows us to reason about automatic differentiation and devise efficient algorithms for the reverse pass.

\subsection{Resolvent}
\label{sec:resolvent}

We start by formalizing how, in the classical deep learning framework, different layers are combined to form a network.
Intuitively, function composition appears to be the natural operation to do so. A sequence of layers
\begin{equation*}
    X_0 \xrightarrow{l_1} X_1 \xrightarrow{l_2} \dots X_{d-1} \xrightarrow{l_d} X_d
\end{equation*}
is composed into a map $X_0 \rightarrow X_d$. We denote composition by juxtaposing functions:
\begin{equation*}
    l_d l_{d-1} \dots l_2 l_1\colon X_0 \rightarrow X_d.
\end{equation*}
However, this intuition breaks down in the case of shortcut connections or more complex, non-sequential architectures.

From a mathematical perspective, a natural alternative is to consider a {\em global} space $X = \bigoplus_{i=0}^d X_i$, and the global endofunction
\begin{equation}\label{eq:global_endofunction}
    f = \sum_{i=1}^d l_i \in C^1(X, X).
\end{equation}
What remains to be understood is the relationship between the function $f$ and the layer composition $l_d l_{d-1} \dots l_2 l_1$. To clarify this relationship, we assume that the output of the network is the entire space $X$, and not only the output of the last layer, $X_d$.
Let the input function be the continuously differentiable inclusion map $g\in C^1(X_0, X)$. The map $g$ embeds the input data into an augmented space, which encompasses input, hidden layers, and output. The network transforms the input map $g$ into an output map $h\in C^1(X_0, X)$. From a practical perspective, $h$ computes the activation values of all the layers and stores not only the final result, but also all the activations of the intermediate layers.

The key observation, on which our framework is based, is that $f$ (the sum of all layers, as in~\cref{eq:global_endofunction}) and $g$ (the input function) alone are sufficient to determine $h$ (the output function). Indeed, $h$ is the only map in $C^1(X_0, X)$ that respects the following property:
\begin{equation}
    \label{eq:machine}
    h = g + f h.
\end{equation}

In summary, we will use~\cref{eq:global_endofunction} to recover neural networks as particular cases of our framework. There, composition of layers is replaced by their sum. Layers are no longer required to be sequential, but they must obey a weaker condition of {\em independence}, which will be discussed in detail in~\cref{sec:composability}. Indeed, \cref{eq:machine} holds also in the presence of shortcut connections, or more complex architectures such as UNet~\cite{liHDenseUNetHybridDensely2018} (see~\cref{fig:complexshortcuts} for a worked example).
The existence of a unique solution to~\cref{eq:machine} for any choice of input function $g$ is the minimum requirement to ensure a well-defined input-output mapping for general architectures. It will be the defining property of a {\em machine}, our generalization of a feedforward deep neural network.

\begin{definition}
    \label{def:machine}
    Let $X$ be a normed vector space. Let $k \in \N \cup \{\infty\}$. An endofunction $f \in C^k(X, X)$ is a \emph{k-differentiable machine} if, for all normed vector space $X_0$ and for all map $g\in C^k(X_0, X)$, there exists a unique map $h\in C^k(X_0, X)$ such that \cref{eq:machine} holds.
    We refer to $X_0$ and $X$ as {\em input space} and {\em machine space}, respectively. We refer to \cref{eq:machine} as the {\em machine equation}.
\end{definition}

In the remainder we assume and shall use $k=1$, in other words machines are 1-differentiable, to allow for backpropagation. However, $k$-differentiable machines, with $k > 1$, could be used to perform gradient-based hyperparameter optimization, as discussed in~\cite{bengio2000gradient,lorraine2020optimizing}. The results shown here for $k=1$ can be adapted in a straightforward way to $k > 1$.

\Cref{def:machine} and \cref{eq:machine} describe the link between the input function $g$ and the output function $h$. By a simple algebraic manipulation, we can see that \cref{eq:machine} is equivalent to
\begin{equation*}
    (\id - f) h = g.
\end{equation*}
In other words, $f$ is a machine if and only the composition with $\id - f$ induces a bijection $C^1(X_0, X) \xrightarrow{\sim} C^1(X_0, X)$ for all normed vector space $X_0$. It is a general fact that this only happens whenever $\id - f$ is an isomorphism, as will be shown in the following proposition. This will allow us to prove that a given function $f$ is a machine by explicitly constructing an inverse of $\id - f$.

\begin{proposition}
    \label{prop:oneminusf}
    Let $X$ be a normed vector space. $f \in C^1(X, X)$ is a machine if and only if $\,\id-f$ is an isomorphism. Whenever that is the case, the {\em resolvent} of $f$ is the mapping
    \begin{equation}\label{eq:resolvent}
        R_f = (\id-f)^{-1}.
    \end{equation}
    Then, $h = g + fh$ if and only if $h = R_f g$. 
\end{proposition}
\begin{proof}
    Let us assume that $f$ is a machine. Let $g = \id$ and $h$ be such that $h = g + f h$. Then, $(\id - f) h = \id$, that is to say $\id - f$ has a right inverse $h$. Let $h_1, h_2$ be such that $(\id - f)h_1 = (\id - f)h_2$. Let $g = (\id - f)h_1 = (\id - f)h_2$. Then,
    \begin{equation*}
        h_1 = g + f h_1
        \quad\text{ and }\quad
        h_2 = g + f h_2,
    \end{equation*}
    hence $h_1 = h_2$, therefore $\id - f$ is injective. As
    \begin{equation*}
        (\id - f) h (\id - f) = \id - f
    \end{equation*}
    and $\id - f$ is injective, it follows that $h (\id - f) = \id$, so $\id - f$ is necessarily an isomorphism.
    Conversely, let us assume that $\id - f$ is an isomorphism. Then, for all normed vector space $X_0$ and for all $g, h \in C^1(X_0, X)$,
    \begin{equation*}
        h = g + f h \text{ if and only if } h = (\id-f)^{-1} g = R_f g.
    \end{equation*}
\end{proof}

Thanks to \cref{prop:oneminusf}, it follows that the derivative of a machine, as well as its dual, are also machines. This will be relevant in the following sections, to perform parameter optimization on machines.

\begin{proposition}\label{prop:dualmachine}
    Let $f \in C^1(X, X)$ be a machine Let $x_0 \in X$. Then, the derivative $Df(x_0)$---a bounded linear endofunction in $B(X, X)$---and its dual $\left(Df(x_0)\right)^* \in B(X^*, X^*)$ are machines, with resolvents
    \begin{equation*}
        R_{Df(x_0)} = DR_{f}(x_0)
        \quad\text{ and }\quad
        R_{\left(Df(x_0)\right)^*} = \left(DR_{f}(x_0)\right)^*,
    \end{equation*}
    respectively. 
\end{proposition}
\begin{proof}
    By differentiating \cref{eq:resolvent}, it follows that
    \begin{equation}\label{eq:resolvent_differential}
        DR_{f}(x_0) =  \left(\id - Df(x_0)\right)^{-1},
    \end{equation}
    hence $Df(x_0)$ is a machine with resolvent $DR_{f}(x_0)$. By taking the duals in \cref{eq:resolvent_differential}, it follows that
    \begin{equation*}
        \left(DR_{f}(x_0)\right)^* =  \left(\id - Df(x_0)^*\right)^{-1},
    \end{equation*}
    hence $\left(Df(x_0)\right)^*$ is a machine with resolvent $\left(DR_{f}(x_0)\right)^*$.
\end{proof}

\subsubsection*{Examples}

Standard sequential neural networks are machines. Let us consider a product of normed vector spaces $X = \bigoplus_{i=0}^d X_i$, and, for each $i \in \range{1}{d}$, a map $l_i\in C^1(X_{i-1}, X_i)$. This is analogous to a sequential neural network. Let $f = \sum_{i=1}^d l_i$. The inverse of $\id - f$ can be constructed explicitly via the following sequence:
\begin{equation*}
    y_0 = x_0
    \quad\text{ and }\quad
    y_i = l_i(y_{i - 1}) + x_i
    \text{ for }
    i \in \range{1}{d}.
\end{equation*}
Then, it is straightforward to verify that
\begin{align*}
    (\id - f) (y_0, y_1, \dots, y_d)
    &= (y_0, y_1 - l_1(y_0), \dots, y_d - l_d(y_{d-1})) \\
    &= (x_0, x_1, \dots, x_d).
\end{align*}
Conversely, let us assume that $x = (\id - l)\x$. Then,
\begin{align*}
    (y_0, y_1, \dots, y_d)
    &= (x_0, l_1(y_0) + x_1, \dots, l_d(y_{d-1}) + x_d) \\
    &= (\x_0, l_1(y_0) + \x_1 - l_1(\x_0), \dots, l_d(y_{d-1}) + \x_d - l_d(\x_{d-1})).
\end{align*}
By induction, for all $i \in \range{0}{d}$, $y_i = \x_i$. Hence,
\begin{equation*}
    R_f(x_0, x_1, \dots, x_d) = (y_0, y_1, \dots, y_d)
\end{equation*}
is the inverse of $\id - f$.

Classical results in linear algebra provide us with a different but related class of examples. Let us consider a bounded linear operator $f \in B(X, X)$. If $f$ is {\em nilpotent}, that is to say, there exists $n \in \N$ such that $f^n = 0$, then $f$ is a machine. The resolvent can be constructed explicitly as
\begin{equation*}
    (\id - f)^{-1} = \id + f + f^2 + \dots + f^{n-1}.
\end{equation*}

The sequential neural network and nilpotent operator examples have some overlap: whenever all layers $l_i$ are linear, $f = \sum_{i=1}^d l_i$ is a linear nilpotent operator. However, in general they are distinct: neural networks can be nonlinear and nilpotent operators can have more complex structures (corresponding to shortcut connections). The goal of the next section is to discuss a common generalization---machines of {\em finite depth}.

\subsection{Depth}
\label{sec:depth}


We noted in \cref{sec:resolvent} that nilpotent continuous linear operators and sequential neural networks are machines, whose resolvents can be constructed explicitly. The same holds true for a more general class of endofunctions, namely endofunctions of {\em finite depth}. In this section, we will give a precise definition of depth and show a procedure to compute the resolvent of endofunctions of finite depth. We follow the convention that, given a normed vector space $X$, a cofiltration is a sequence of quotients
\begin{equation*}
    X / V_i \rightarrow X / V_j
    \quad\text{ for }\quad
    i \ge j,
\end{equation*}
where each $V_i$ is a closed subspace of $X$.

\begin{definition}\label{def:depth}
    Let $X$ be a normed vector space and $f \in C^1(X, X)$. Let $d \in \N$. A sequence of closed vector subspaces
    \begin{equation*}
       X \supseteq V_0 \supseteq V_1 \supseteq \dots \supseteq V_d = 0
    \end{equation*}
    with associated projections $\pi_i \colon X \rightarrow X / V_i$ is a {\em depth cofiltration} of length $d$ for $f$ if the following conditions are verified.
    \begin{itemize}
        \item $\pi_0 f = 0$ or, equivalently, $\img f \subseteq V_0$.
        \item For all $i$ in $\range{1}{d}$, there exists $\f_i \in C^1(X / V_{i-1}, X / V_{i})$ such that
        \begin{equation*}
            \pi_i f = \f_i \pi_{i-1}.
        \end{equation*}
    \end{itemize}
    The {\em depth} of $f$ is the length of its shortest depth cofiltration, if any exists, and $\infty$ otherwise.
\end{definition}

\begin{remark}\label{rm:span_image}
Even though it is not required that $V_0 = \overline{\spn(\img f)}$, it is always possible, given a depth cofiltration $V_0, \dots, V_n$, to construct a new depth cofiltration
\begin{equation*}
    \V_i = V_i \cap \overline{\spn(\img f)}.
\end{equation*}
Since $\overline{\spn(\img f)} \subseteq V_0$, then $\V_0 = \overline{\spn(\img f)}$. This will be useful when combining depth cofiltrations to build depth cofiltrations of more complex endofunctions, as in \cref{thm:sum_of_machines}.
\end{remark}

\begin{proposition}\label{prop:depth_differential}
    Let $f \in C^1(X, X)$ be a machine. A sequence of closed vector subspaces
\begin{equation*}
    X \supseteq V_0 \supseteq V_1 \supseteq \dots \supseteq V_d = 0
\end{equation*}
is a depth cofiltration for $f$ if and only if it is a depth cofiltration for $Df(x_0)$ for all $x_0 \in X$. Whenever that is the case,
\begin{equation*}
    X^* \supseteq \left(X / V_{d-1}\right)^* \supseteq \dots \supseteq \left(X / V_{0}\right)^* \supseteq \left(X / X\right)^* = 0
\end{equation*}
is a depth cofiltration for $\left(Df(x_0)\right)^*$ for all $x_0 \in X$.
\end{proposition}
\begin{proof}
    The claim concerning the differential machine follows by \cref{prop:quotient} in \cref{sec:normedvectorspaces}.
\end{proof}

In simple cases, depth cofiltrations can be computed directly. For instance, if $f$ is linear and continuous, then
\begin{equation*}
    X \supseteq \ker f^d \supseteq \dots \supseteq \ker f \supseteq \ker f^0 = 0
\end{equation*}
is a depth cofiltration for $f$ if $f^{d+1} = 0$. Conversely, if a continuous linear operator $f$ admits a depth cofiltration of length $d$, then necessarily $f^{d+1} = 0$. Hence, a continuous linear operator has finite depth if and only if it is nilpotent.

Sequential neural networks are another example of endofunction of finite depth. Let us consider a sequential architecture
\begin{equation*}
    X_0 \xrightarrow{l_1} X_1 \xrightarrow{l_2} \dots X_{d-1} \xrightarrow{l_d} X_d
\end{equation*}
and
\begin{equation*}
    l = \sum_{i=1}^n l_i \in C^1(X, X), \text{ where } X = X_0 \oplus \dots \oplus X_d.
\end{equation*}
Naturally $V_i = X_{i+1} + \dots + X_d$ defines a depth cofiltration for $l$. A similar result holds for {\em acyclic} neural architectures with arbitrarily complex shortcuts. However, proving that directly is nontrivial: it will become much more straightforward with the tools developed in \cref{sec:composability}. For now, we will simply assume that a given endofunction $f$ has finite depth, and we will show how to compute its resolvent.

\begin{definition}\label{def:depthsequence}
    Let $f \in C^1(X, X)$ and $g \in C^1(X_0, X)$. Let $d \in \N$ and let
    \begin{equation*}
        X \supseteq V_0 \supseteq V_1 \supseteq \dots \supseteq V_d = 0
    \end{equation*}
     be a depth cofiltration. Its associated {\em depth sequence} is defined as follows:
     \begin{equation*}
         \h_0 = \pi_0 g
         \quad\text{ and }\quad
         \h_{i} = \pi_{i} g + \f_{i} \h_{i - 1} \text{ \;for } i \in \range{1}{d},
     \end{equation*}
     where for all $i \in \range{1}{d}$, $\h_i \in C^1(X_0, X / V_i)$. A sequence $\range{h_0}{h_d} \subseteq C^1(X_0, X)$ is a {\em lifted depth sequence} if $\pi_i h_i = \h_i$ for all $i \in \range{0}{d}$.
\end{definition}

In other words, a depth sequence is a sequence of functions that approximate more and more accurately a solution of the machine equation, as we will prove in the following theorem. In general, the depth sequence can be lifted in different ways, which correspond to algorithms to solve the machine equation of different computational efficiency, as shown in \cref{fig:depthsequence}.

\begin{figure}
    \centering
    \includegraphics[width=360pt]{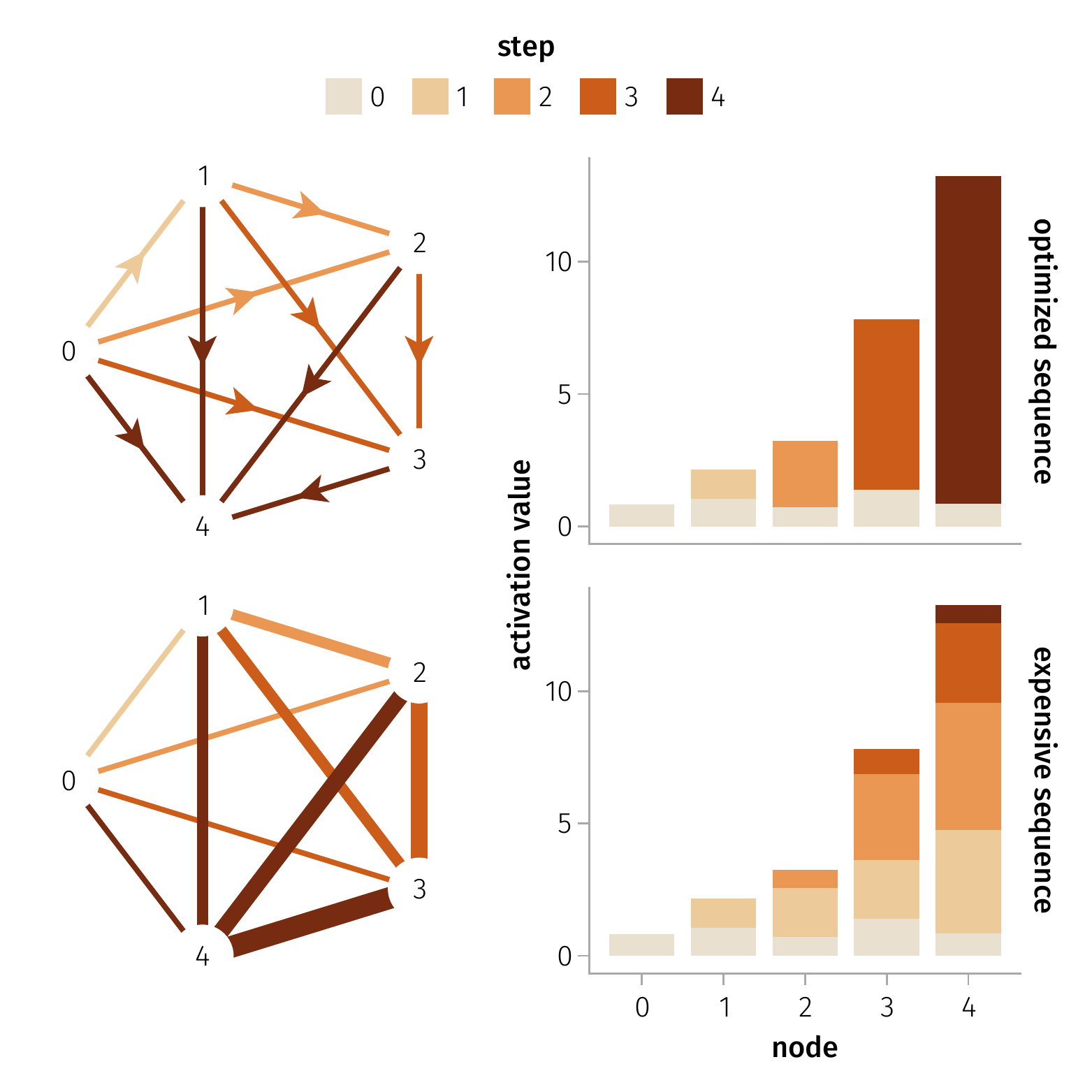}
    \caption{\textbf{Different sequences to solve a linear machine with shortcuts on $5$ nodes.}
    The left column shows the connectivity graph, whereas the right column describes the accumulated activation value on each node for a given input.
    For visual simplicity, we take positive connectivity matrix and input values, so that all updates are positive and can be represented concisely in a stacked bar plot. Two possible approaches to solve the machine equation are exemplified. The top row represents an efficient strategy: at step $i$ we update only the $i$-th node, as implied by color and orientation of the edges. In the bottom row, at the $i$-th step we evaluate $f h_i + g$. This is inefficient, as some connections need to be recomputed several times, as encoded by the line width of the edges of the bottom graph. The optimized sequence avoids this inefficiency by deferring the computation of each connection until the input value is fully determined.} \label{fig:depthsequence}
\end{figure}

\begin{proposition}\label{prop:lifted_depthsequence}
    Let $\phi \in C^1(X_0, V_0)$. The sequence
    \begin{equation*}
        h_0^\phi = g + \phi
        \quad\text{ and }\quad
         h_{i}^\phi = g + fh_{i-1}^\phi \text{ for } i \in \range{1}{d}
    \end{equation*}
    is a lifted depth sequence.
\end{proposition}
\begin{proof}
    For $i = 0$, $\pi_0 h_0^\phi = \pi_0 (g + \phi) = \pi_0 g = \h_0$. If $\pi_{i-1} h_{i-1}^\phi = \h_{i-1}^\phi$, then 
    \begin{equation*}
        \pi_i h_i^\phi = \pi_i(g + fh_{i-1}^\phi) = \pi_i g + \pi_i fh_{i-1}^\phi = \pi_i g + \f_i \pi_{i-1} h_{i-1}^\phi = \pi_g + \f_i \h_{i-1} = \h_i,
    \end{equation*}
    hence the claim follows by induction.
\end{proof}

\begin{theorem}
    \label{thm:machine_of_finite_depth}
    Let us assume that $f \in C^1(X, X)$ admits a depth cofiltration of length $d$. Then, $f$ is a machine. Furthermore, let us consider $g \in C^1(X_0, X)$, and let $\h$ be its depth sequence. Then, $\h_d = g + f \h_d$.
\end{theorem}
\begin{proof}
    Let $h_i^\phi$ be as in \cref{prop:lifted_depthsequence}. Then, $h_d^0 = h_d^{fg} = h_{d+1}^0 = fg + h_d^0$, hence $h_d^0 = \h_d$ solves the machine equation. To see uniqueness, let $h$ be a solution to $h = fg + h$. Then, for all $i \in \range{0}{d}$, $h = h_i^{fh}$,  hence $h_d^0 = h_d^{fh} = h$.
\end{proof}

\subsection{Composability}\label{sec:composability}

Here, we develop the notion of machine {\em independence}, which will be crucial for {\em composability}, as it will allow us to create complex machines as a sum of simpler ones. In particular, we will show that deep neural networks can be decomposed as sum of layers and are therefore machines of finite depth.

\begin{definition}
    \label{def:independence}
    Let $X$ be a normed vector space. Let $f_1, f_2 \in C^1(X, X)$. We say that $f_1$ \emph{does not depend} on $f_2$ if, for all $x_1, x_2 \in X$, and for all $\lambda \in \R$, the following holds:
    \begin{equation}
        \label{eq:independence}
        f_1 (x_1 + \lambda f_2(x_2)) = f_1(x_1).
    \end{equation}
    Otherwise, we say that $f_1$ depends on $f_2$.
\end{definition}

\Cref{def:independence} is quite useful to compute resolvents. For instance, $f$ does not depend on itself if and only if it has depth at most $1$, in which case it is a machine, and its resolvent can be computed via $R_f = \id + f$. Furthermore, by combining machines of finite depth with appropriate independence conditions, we again obtain machines of finite depth.

If $f_1$ is linear, then $f_1$ does not depend on $f_2$ if and only if $f_1 f_2 = 0$, but in general the two notions are distinct. For instance, the following pair of functions
\begin{equation*}
    f_1(x) = x-3
    \quad\text{ and }\quad
    f_2(x) = 3
\end{equation*}
respects $f_1 f_2 = 0$, but $f_1$ depends on $f_2$ as $x - 3\lambda \neq x$ for $\lambda \neq 0$.

It follows from \cref{prop:quotient} that \cref{def:independence} has some alternative formulations. $f_1$ does not depend on $f_2$ if and only if it factors through the following quotient:
\begin{equation*}
    \begin{tikzcd}[column sep=small]
        X \arrow[swap]{rd}{\pi}  \arrow{rr}{f_1}&
        & X \\
        & X / \overline{\spn(\img f_2)} \arrow[dotted]{ru} 
        &
    \end{tikzcd}
\end{equation*}
That is equivalent to requiring that at all points the differential of $f_1$ factors via $\pi$, that is to say
\begin{equation}\label{eq:independence_differential}
    (Df_1(x_1))f_2(x_2) = 0 \text{ for all } x_1, x_2 \in X. 
\end{equation}
Given $f_1, f_2 \in C^1(X, X)$, the sets
\begin{equation*}
    \{ f \in C^1(X, X) \, | \, D(f_1(x_1))f(x_2) = 0  \text{ for all } x_1, x_2 \in X \}
\end{equation*}
and
\begin{equation*}
    \{ f \in C^1(X, X) \, | \, D(f(x_1))f_2(x_2) = 0  \text{ for all } x_1, x_2 \in X \}
\end{equation*}
are vector spaces, as they are the intersection of kernels of linear operators. In other words, if $f_1$ does not depend on $f_2$ and $\hat f_2$, then it also does not depend on $\lambda f_2 + \hat \lambda \hat f_2$, and if $f_1$ and $\hat f_1$ do not depend on $f_2$, then neither does $\lambda f_1 + \hat \lambda \hat f_1$.

\begin{theorem}
    \label{thm:sum_of_machines}
    Let $f_1, f_2$ be machines, of depth $d_1, d_2$ respectively, such that $f_1$ does not depend on $f_2$. Then $f_1+f_2$ is also a machine of depth $d \le d_1 + d_2$ and $R_{f_1+f_2} = R_{f_2} R_{f_1}$. If furthermore $f_2$ does not depend on $f_1$, then $R_{f_1+f_2} = R_{f_1} + R_{f_2} - \id$ and $d \le \max(d_1, d_2)$.
\end{theorem}
\begin{proof}
    By~\cref{prop:oneminusf,eq:independence}, $f_1+f_2$ is a machine:
    \begin{equation}
        \label{eq:compositionindependent}
        (\id - f_1)(\id - f_2) = (\id - f_1 - f_2),
    \end{equation}
    so $(\id - f_1 - f_2)$ is an isomorphism (composition of isomorphisms). \Cref{eq:compositionindependent} also determines the resolvent:
    \begin{equation*}
        R_{f_1+f_2} = (\id - f_1 - f_2)^{-1} = (\id - f_2)^{-1}(\id - f_1)^{-1} = R_{f_2}R_{f_1}.
    \end{equation*}
    Moreover, if $f_2$ does not depend on $f_1$, then
    \begin{align*}
        f_1(R_{f_1} + R_{f_2} - \id) = f_1(R_{f_1} + f_2 R_{f_2}) = f_1 R_{f_1} = R_{f_1} - \id,\\
        f_2(R_{f_1} + R_{f_2} - \id) = f_2(f_1 R_{f_1} + R_{f_2}) = f_2 R_{f_2} = R_{f_2} - \id.
    \end{align*}
    Hence,
    \begin{equation*}
        R_{f_1} + R_{f_2} - \id = \id + (f_1+f_2)(R_{f_1} + R_{f_2} - \id).
    \end{equation*}

    To prove the bounds on $d$, we can assume that $d_1$ and $d_2$ are finite, otherwise the claim is trivial. Let $X \supseteq V_0^1 \supseteq \dots \supseteq V_{d_1}^1 = 0$ and $X \supseteq V_0^2 \supseteq \dots \supseteq V_{d_2}^2 = 0$ be depth cofiltrations of minimal length for $f_1$ and $f_2$ respectively. By \cref{rm:span_image}, we can choose them such that
    \begin{equation*}
        V_0^1 = \overline{\spn(\img f_1)}
        \quad\text{ and }\quad
        V_0^2 = \overline{\spn(\img f_2)}.
    \end{equation*}
    If $f_1$ does not depend on $f_2$, then
    \begin{equation*}
        X \supseteq V_0^1 + V_0^2 \supseteq \dots \supseteq V_{d_1 - 1}^1 + V_0^2  \supseteq V_0^2 \supseteq \dots \supseteq V_{d_2}^2 = 0
    \end{equation*}
    is a depth cofiltration of length $d_1 + d_2$ for $f_1 + f_2$. If also $f_2$ does not depend on $f_1$, then we can set $d = \max(d_1, d_2)$ and define
    \begin{equation*}
        X \supseteq V_0^1 + V_0^2 \supseteq V_1^1 + V_1^2 \dots \supseteq V_{d-1}^1 + V_{d-1}^2  \supseteq V_{d}^1 + V_{d}^2 = 0,
    \end{equation*}
    where by convention $V_i^1 = 0$ if $i > d_1$ and $V_i^2 = 0$ if $i > d_2$.
\end{proof}

\begin{remark}
    The depth inequality, that is to say if $f_1$ does not depend on $f_2$, then the depth of the sum of $f_1$ and $f_2$ is bounded by the sum of the depths, is a nonlinear equivalent of an analogous result in linear algebra. Namely, given $L_1, L_2$ nilpotent operators with $L_1L_2 = 0$, the sum $L_1 + L_2$ is also nilpotent, and if $L_1 ^ {n_1} = L_2 ^ {n_2} = 0$, then $(L_1 + L_2) ^ {n_1 + n_2 - 1} = 0$.
\end{remark}

\begin{figure}
    \begin{tikzpicture}[scale=0.8, every node/.style={scale=0.8}]
        \node[emptyvertex] (v1) {};
        \node[emptyvertex,below right of=v1] (v3) {};
        \node[emptyvertex,below left of=v3] (v2) {};
        \node[emptyvertex,right of=v3] (v4) {};
        \node[emptyvertex,right of=v4] (v6) {};
        \node[emptyvertex,above right of=v4] (v5) {};
        \node[emptyvertex,below right of=v4] (v7) {};
        \node[emptyvertex,right of=v6] (v8) {};

        \node[vertex] at (v1) {};
        \node[vertex] at (v2) {};
        \node[vertex] at (v3) {};
        \node[vertex] at (v4) {};
        \node[vertex] at (v5) {};
        \node[vertex] at (v6) {};
        \node[vertex] at (v7) {};
        \node[vertex] at (v8) {};

        \node[] at (v1) {\(X_1\)};
        \node[] at (v2) {\(X_2\)};
        \node[] at (v3) {\(X_3\)};
        \node[] at (v4) {\(X_4\)};
        \node[] at (v5) {\(X_5\)};
        \node[] at (v6) {\(X_6\)};
        \node[] at (v7) {\(X_7\)};
        \node[] at (v8) {\(X_8\)};

      \begin{pgfonlayer}{background}
          \draw[edge,color=sienna1] (v1) -- (v2) -- (v3) -- (v1);
          \draw[edge,color=sienna4] (v4) -- (v5) -- (v6) -- (v7) -- (v4) -- (v6);
          \draw[edge,color=sienna2] (v1) -- (v5);
          \draw[edge,color=sienna3] (v3) -- (v4);
          \draw[edge,color=sienna5] (v6) -- (v8);
      \end{pgfonlayer}

      \node[elabel,color=sienna1,label=right:\(f_1: X_1 \times X_2\rightarrow X_3\)]  (e1) at (-7.2,0.4) {};
      \node[elabel,below of=e1,color=sienna2,label=right:\(f_2: X_1\rightarrow X_5\)]  (e2) {};
      \node[elabel,below of=e2,color=sienna3,label=right:\(f_3: X_3\rightarrow X_4\)]  (e3) {};
      \node[elabel,below of=e3,color=sienna4,label=right:\(f_4: X_4\rightarrow X_5 \times X_6 \times X_7\)]  (e4) {};
      \node[elabel,below of=e4,color=sienna5,label=right:\(f_5: X_6\rightarrow X_8\)]  (e5) {};
      \end{tikzpicture}
    \caption{\textbf{Graphical representation of a neural network with complex shortcuts as sum of machines of depth $1$.} This graphical representation corresponds to the neural network mapping $(x_1,\; x_2,\; x_3,\; x_4,\; \dots, x_8)$ to $(y_1,\; y_2,\; y_3,\; y_4,\; \dots, y_8)$ via layers $\range{f_1}{f_5}$.}\label{fig:complexshortcuts}
    \begin{minipage}{\linewidth}
        Explicitly, output values are computed as follows:
        \begin{align*}
            y_1 &= x_1\\
            y_2 &= x_2\\
            y_3 &= f_1(x_1, x_2) + x_3\\
            y_4 &= f_3(f_1(x_1, x_2) + x_3) + x_4\\
            y_5 &= f_2(x_1) + \pi_{X_5}f_4(f_3(f_1(x_1, x_2) + x_3) + x_4) + x_5\\
            y_6 &= \pi_{X_6} f_4(f_3(f_1(x_1, x_2) + x_3) + x_4) + x_6\\
            y_7 &= \pi_{X_7} f_4(f_3(f_1(x_1, x_2) + x_3) + x_4) + x_7\\
            y_8 &= f_5(\pi_{X_6} f_4(f_3(f_1(x_1, x_2) + x_3) + x_4) + x_6) + x_8
        \end{align*}
    \end{minipage}
\end{figure}
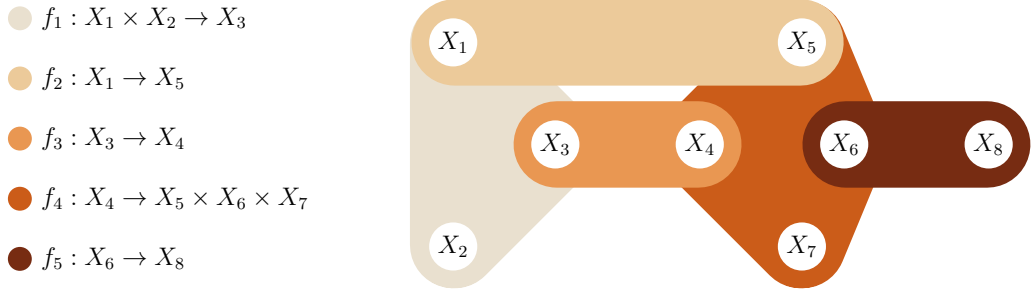

A natural notion of {\em architecture with shortcuts} follows from \cref{thm:sum_of_machines}. Let $f_1, \dots, f_n$ be such that $f_i$ does not depend on $f_j$ if $i \le j$. Then each $f_i$ has depth at most $1$, hence $f = \sum_{i=1}^n f_i $ has depth at most $n$, by \cref{thm:sum_of_machines}. Indeed, $f_{1} + \dots + f_{i-1}$ does not depend on $f_{i}$, as can be verified for each addend individually thanks to \cref{eq:independence_differential}, hence by induction $f_1 + \dots + f_i$ has depth at most $i$. Then, $f$ is a machine of depth at most $n$, whose resolvent can be computed as
\begin{equation*}
    R_{f_n}\cdots R_{f_1} g = (\id + {f_n})\cdots (\id + {f_1}) g.
\end{equation*}
In practice, this corresponds to the lifted depth sequence
\begin{equation*}
    \h_0 = g
    \quad\text{ and }\quad
    \h_{i+1} = \h_i + f_i \h_i.
\end{equation*}
This strategy can be applied to {\em acyclic} architectures with arbitrarily complex shortcuts, as illustrated in \cref{fig:complexshortcuts}. The architecture described there has depth at most $4$, as the endofunctions $f_1, f_2 + f_3, f_4, f_5$ all have depth at most $1$, and each of them does not depend on the following ones.

More generally, \cref{thm:sum_of_machines} establishes a clear link between sums of independent machines and compositions of layers in classical feedforward neural networks. The independence condition determines the order in which machines should be concatenated, even in the presence of complex shortcut connections. Furthermore, if the initial building blocks all have finite depth, then so does the sum. Thus, we can compute the machine's resolvent efficiently. As a consequence, machines of finite depth are a {\em practically computable} generalization of deep neural networks and nilpotent operators.

\subsection{Optimization}
\label{sec:optimization}

The ability to minimize an error function is crucial in machine learning applications. This section is devoted to translating classical backpropagation-based optimization to our framework. Given the input map $g\colon X_0\rightarrow X$ and a loss function $\mathcal{L}\colon X\rightarrow\mathbb{R}$, we wish to find $f\colon X\rightarrow X$ such that the composition $\mathcal{L}h$ is minimized. To constrain the space of possible endofunctions (architectures and weights), we restrict the choice of $f$ to a smoothly parameterized family of functions $f_p$, where $p$ varies within a parameter space $P$.

\paragraph{Parametric machines.}
Let $P$ be a normed vector space of {\em parameters}. A {\em parametric machine} is a $C^1$ family of machines $f(p, x) \colon P \times X \rightarrow X$ such that, given a $C^1$ family of input functions $g(p, x_0)$, the family of resolvents $h(p, x_0)$ is also jointly $C^1$ in both arguments. We call $f$ a {\em parametric machine}, with {\em parameter space} $P$. Whenever $f$ is a parametric machine, we denote by $R_f$ its {\em parametric resolvent}, that is the only function in $C^1(P \times X, X)$ such that
\begin{equation*}
    R_f(p, x_0) = x_0 + f(p, R_f(p, x_0)).
\end{equation*}

In practical applications, we are interested in computing the partial derivatives of the parametric resolvent function $R_f$ with respect to the parameters and the inputs. This can be done using the derivatives of $f$ and a resolvent computation. Therefore, the structure and cost of the backward pass (backpropagation) are comparable to those of the forward pass. We recall that the backward pass is the computation of the dual operator of the derivative of the forward pass.

\begin{theorem}\label{thm:derivatives}
Let $f(p, x)$ be a parametric machine. Let $R_f$ denote the parametric resolvent mapping
\begin{equation*}
    x = R_f(p, x_0).
\end{equation*}
Then, the following equations hold:
\begin{equation}\label{eq:forward_mode_machine}
    \frac{\partial R_f}{\partial x_0} = R_{\frac{\partial f}{\partial x}}
    \quad\text{ and }\quad
    \frac{\partial R_f}{\partial p} = \frac{\partial R_f}{\partial x_0} \frac{\partial f}{\partial p}.
\end{equation}
Analogously, by considering the dual of each operator,
\begin{equation}\label{eq:reverse_mode_machine}
    \left(\frac{\partial R_f}{\partial x_0}\right)^* = \left(R_{\frac{\partial f}{\partial x}}\right)^*
    \quad\text{ and }\quad
    \left(\frac{\partial R_f}{\partial p}\right)^* = \left(\frac{\partial f}{\partial p}\right)^* \left(\frac{\partial R_f}{\partial x_0}\right)^*.
\end{equation}
In other words,
\begin{itemize}
    \item the partial derivative of $R_f$ with respect to the inputs can be obtained via a resolvent computation, and
    \item the partial derivative of $R_f$ with respect to the parameters is the composition of the partial derivative of $R_f$ with respect to the inputs and the partial derivative of $f$ with respect to the parameters.
\end{itemize}
\end{theorem}
\begin{proof}
We can differentiate $R_f$ with respect to $p$ and $x_0$ by differentiating the machine equation $x = x_0 + f(p, x)$. Explicitly,
\begin{equation*}
    \frac{\partial R_f}{\partial x_0}
    = \left(\id - \frac{\partial f}{\partial x}\right)^{-1}
    = R_{\frac{\partial f}{\partial x}}
    \text{ and }
    \frac{\partial R_f}{\partial p}
    = \left(\id - \frac{\partial f}{\partial x}\right)^{-1} \frac{\partial f}{\partial p}
    = \frac{\partial R_f}{\partial x_0} \frac{\partial f}{\partial p}.
\end{equation*}
\Cref{eq:reverse_mode_machine} follows from \cref{eq:forward_mode_machine} by duality.
\end{proof}

The relevance of \cref{thm:derivatives} is twofold. On the one hand, it determines a {\em practical approach} to backpropagation for general parametric machines. Initially the resolvent of $\left(\frac{\partial f}{\partial x}\right)^*$ is computed on the gradient of the loss function $\mathcal{L}$. Then, the result is backpropagated to the parameters. In symbols,
\begin{equation*}
    \frac{\partial \mathcal{L}(R_f(p, x_0))}{\partial p} =
    \left(\frac{\partial f}{\partial p}\right)^* \left(\frac{\partial R_f}{\partial x_0}\right)^* D\mathcal{L}(R_f(p, x_0)).
\end{equation*}
The gradient $D\mathcal{L}(x)$, where $x = R_f(p, x_0)$, linearly maps tangent vectors of $X$ to scalars and is therefore a cotangent vector of $X$. Indeed, the dual machine $\left(\frac{\partial f}{\partial x}\right)^*$ is an endofunction of the cotangent space of $X$.
On the other hand, \cref{thm:derivatives} guarantees that in a broad class of practical cases the computational complexity of the backward pass is comparable to the computational complexity of the forward pass. We will show this practically in the following section.

\section{Implementation and performance}
\label{sec:implementation}

In this section, we shall analyze several standard and non-standard architectures in the machine framework, provide a general implementation strategy, and discuss memory usage and performance for both forward and backward pass.
We consider a broad class of examples where $f$ has both a linear component $w_p$ (parametrized by $p$) and a nonlinear component $\sigma$. Different choices of $w$ will correspond to different architecture (multi-layer perceptron, convolutional neural network, recurrent neural network) with or without shortcuts.

We split the space $X$ as a direct sum $X = Y \oplus Z$, i.e., $x = (y, z)$, where $y$ and $z$ correspond to values before and after the nonlinear activation function, respectively. Hence, we write $f_p = w_p + \sigma$, with
\begin{equation*}
    \sigma \colon Y \rightarrow Z
    \quad\text{ and }\quad
    w_p \colon Z \rightarrow Y.
\end{equation*}
The machine equation
\begin{equation*}
    x = f_p(x) + x_0
\end{equation*}
can be written as a simple system of two equations:
\begin{equation*}
    y = w_p z + y_0
    \quad\text{ and }\quad
    z = \sigma(y) + z_0.
\end{equation*}
Given cotangent vectors $u_0 \in Z^*, v_0 \in Y^*$ (which are themselves computed by backpropagating the loss on the machine output) we can run the following {\em dual machine}:
\begin{equation*}
    u = w_p^* v + u_0
    \quad\text{ and }\quad
    v = (D\sigma (y))^* u + v_0.
\end{equation*}
Then, \cref{eq:reverse_mode_machine} boils down to the following rule to backpropagate $(v_0, u_0)$ both to the input and the parameter space.
\begin{equation*}
    \left(\frac{\partial x}{\partial x_0}\right)^* (v_0, u_0) = (v, u),
    \quad\text{ and }\quad
    \left(\frac{\partial x}{\partial p}\right)^* (v_0, u_0) = \left(\frac{\partial w_p}{\partial p}\right)^* v.
\end{equation*}
In practical cases, the computation of the dual machine has not only the same structure, but also the same computational complexity of the forward pass. In particular, in the cases we will analyze, the {\em global} linear operator $w_p \in B(Y, Z)$ will be either a fully-connected or a convolutional layer, hence the dual $w_p^*$ would be a fully-connected or a transpose convolutional layer respectively, with comparable computational cost, as shown practically in~\cref{fig:benchmarks} (see \cref{tbl:benchmarks} for the exact numbers). In our applications, the nonlinearity $\sigma$ will be pointwise, hence the derivative $D\sigma(x)$ can be computed pointwise, again with comparable computational cost to the computation of $\sigma$. Naturally, for $\sigma$ to act pointwise, we require that $Y \simeq Z \simeq \R^I$ for some index set $I$.

The first obstacle in defining a machine of the type $w_p + \sigma$ is practical. How should one select a linear operator $w_p$ and a pointwise nonlinearity $\sigma$, under the constraint that $w_p + \sigma$ is a machine of finite depth? We adopt a general strategy, starting from classical existing layers and partitions on index spaces. We take $l_p$ to be a linear operator (in practice, a convolutional or fully connected layer). We consider a partition $I = \bigsqcup_{i=0}^n I_i$ of the underlying index set $I$. For $i \in \range{0}{n}$, let $\pi^Y_i, \pi^Z_i$ be the projection from $Y$ or $Z$ to the subspace corresponding to $I_0 \sqcup \dots \sqcup I_{i}$. We can define the linear component of the machine as follows:
\begin{equation*}
    w_p = \sum_{i=1}^n \left(\pi^Y_{i} - \pi^Y_{i-1}\right) l_p \pi^Z_{i-1},
\end{equation*}
that is to say, it is a modified version of $l_p$ such that outputs in index subsets depend only on inputs in previous index subsets. It is straightforward to verify that
\begin{align*}
    X = Y \oplus Z
    &\supseteq \ker \pi^Y_0 + Z \\
    &\supseteq \ker \pi^Y_0 + \ker \pi^Z_0 \\
    &\supseteq \ker \pi^Y_1 + \ker \pi^Z_0 \\
    &\supseteq \ker \pi^Y_1 + \ker \pi^Z_1 \\
    &\vdots \\
    &\supseteq \ker \pi^Y_n + \ker \pi^Y_n = 0
\end{align*}
is a depth cofiltration for $w_p + \sigma$, hence $w_p + \sigma$ is a machine of depth at most $2n + 1$.

\begin{figure}
    \centering
    \includegraphics[width=340pt]{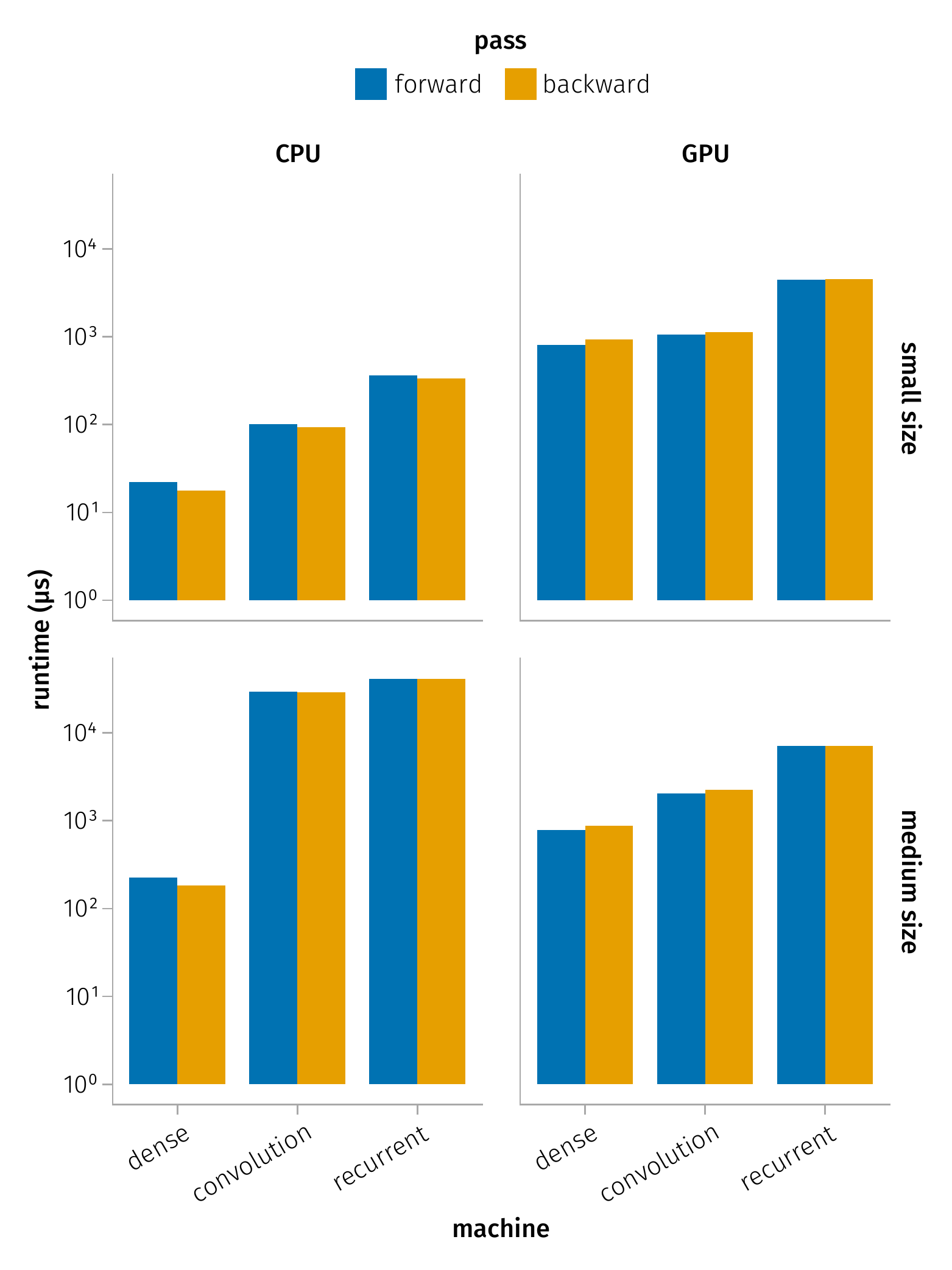}
    \caption{\textbf{Ratio of runtime of backward pass over forward pass.} The runtimes of backward and forward pass are comparable, across different models, problem sizes, and devices. The computation of the backward pass assumes that the forward pass has been computed already, and that its result is available. The backward pass denotes the backpropagation of cotangent vectors from machine space to input space. Backpropagating to parameter space requires an extra operation (see e.g. \cref{eq:resolvent_mlp_q} for the dense case).} \label{fig:benchmarks}
\end{figure}

\subsubsection*{Generalized multi-layer perceptron}

Let us consider a generalization of the multi-layer perceptron in our framework.
Let $x[c]$ (a point in machine space) be a tensor with one index, where $c \in \range{1}{n_c}$. Let $I_0, \dots, I_n$ be a partition of $\range{1}{n_c}$. We adapt the notation of the previous section: whenever possible, capital letters denote tensors corresponding to linear operators in lower case. Let $L[c_2, c_1]$ be a tensor with two indices $c_1, c_2 \in \range{1}{n_c}$, let
\begin{equation*}
    W = \sum_{i = 1}^n \left(\pi^Y_{i} - \pi^Y_{i-1}\right) L \pi^Z_{i-1},
\end{equation*}
and let $\sigma$ a pointwise nonlinearity. We consider the machine equation
\begin{align}
        z &= \sigma(y) + z_0, \label{eq:resolvent_mlp_z}\\
        y &= W z + y_0.\label{eq:resolvent_mlp_y}
\end{align}
The backward pass can be computed via the dual machine computation
\begin{align}
    v &= \sigma'(y) \odot u + v_0, \label{eq:resolvent_mlp_v}\\
    u &= W^* v + u_0, \label{eq:resolvent_mlp_u}
\end{align}
where $\sigma'$ is the derivative of $\sigma$ and $\odot$ is the Hadamard (elementwise) product, and the equations
\begin{equation}
    Q = \sum_{i = 1}^n \left(\pi^Y_{i} - \pi^Y_{i-1}\right) v z^* \pi^X_{i-1}, \label{eq:resolvent_mlp_q}
\end{equation}
where $Q$ represents the cotangent vector $(v_0, u_0)$ backpropagated to the parameters $W$. \Cref{eq:resolvent_mlp_z,eq:resolvent_mlp_y,eq:resolvent_mlp_v,eq:resolvent_mlp_u,eq:resolvent_mlp_q} can be solved efficiently following the procedure described in \cref{alg:mlp}. We describe the procedure exclusively for generalized multi-layer perceptrons, but the equivariant case (convolutional and recurrent neural networks) is entirely analogous.

\begin{algorithm}
	\caption{Computation of non-equivariant machine.}
    \label{alg:mlp}
    \bigskip
    Forward pass:
	\begin{algorithmic}[1]
        \State Initialize arrays $y, z$ of size $n_c$ and value $y = y_0, \, z = z_0$
		\For {$i = 0$ to $n$}
        \State Set $y[I_i] \pluseq W[I_i, :] z$, \cref{eq:resolvent_mlp_y}
        \State Set $z[I_i] \pluseq \sigma\left(y[I_i]\right)$, \cref{eq:resolvent_mlp_z}
		\EndFor
	\end{algorithmic}
    \bigskip
    Backward pass:
    \begin{algorithmic}[1]
        \State Initialize arrays $u, v$ of size $n_c$ and value $u = u_0, \, v = v_0$
		\For {$i = n$ to $0$}
        \State Set $u[I_i] \pluseq \left(L[:, I_i]\right)^* v$, \cref{eq:resolvent_mlp_u}
        \State Set $v[I_i] \pluseq \sigma'\left(y[I_i]\right) \odot u[I_i]$, \cref{eq:resolvent_mlp_v}
		\EndFor
        \State Initialize $Q = v z^*$, \cref{eq:resolvent_mlp_q}
        \State Set $Q[I_j, I_i] = 0$, for all $j \le i$, \cref{eq:resolvent_mlp_q}
	\end{algorithmic}
\end{algorithm}

\subsubsection*{Equivariant architectures}

We include under the broad term {\em equivariant} architectures~\cite{bergomiTopologicalGeometricalTheory2019} all machines whose underlying linear operator $w_p$ is {\em translation-equivariant}---a shift in the input corresponds to a shift in the output. This includes convolutional layers for temporal or spatial data, as well as recurrent neural networks, if we consider the input as a time series that can be shifted forward or backward in time. The similarity between one-dimentional convolutional neural networks and recurrent neural networks will become clear in the machine framework. Both architectures can be implemented with the same linear operator $l_p$ but different index space partitions.

The equivariant case is entirely analogous to the non-equivariant one. We consider the simplest scenario: one-dimensional convolutions of stride one for, e.g., time series data. We consider a discrete grid with two indices
\begin{equation*}
    t \in \range{1}{n_t}, \quad c \in \range{1}{n_c},
\end{equation*}
referring to time and channel, respectively. Thus, the input data will be a tensor of two indices, $y[t, c]$. The convolutional kernel will be a tensor of three indices, $L[\tau, c_1, c_2]$, representing time lag (kernel size), input channel, and output channel, respectively. Let $I_0, \dots, I_n$ be a partition of $\range{1}{n_t} \times \range{1}{n_c}$.

We again denote
\begin{equation*}
    W = \sum_{i = 1}^n \left(\pi^Y_{i} - \pi^Y_{i-1}\right) L \pi^Z_{i-1}
\end{equation*}
and consider the machine equation
\begin{align*}
        z &= \sigma(y) + z_0,\\
        y &= W * z + y_0.
\end{align*}
where $*$ denotes convolution. The backward pass can be computed via the dual machine computation
\begin{align*}
    v &= \sigma'(y) \odot u,\\
    u &= W *^t v + u_0,
\end{align*}
where $*^t$ denotes transposed convolution, and the equations
\begin{align*}
    \hat Q[\tau, c_1, c_2] &= \sum_{t=\tau+1}^{n_t} z[t - \tau, c_1] v[t, c_2],\\
    Q &= \sum_{i = 1}^n \left(\pi^Y_{i} - \pi^Y_{i-1}\right) \hat Q \pi^X_{i-1},\\
\end{align*}
where $Q$ represents the cotangent vector $u_0$ backpropagated to the parameters.

\paragraph{A common generalization of convolutional and recurrent neural networks.} Specific choices of the partition $I_1, \dots, I_n$ will give rise to radically different architectures. In particular, setting $I_i = \range{1}{n_t} \times J_i$ for some partition $J_0 \sqcup \dots \sqcup J_n = \range{1}{n_c}$ gives a deep convolutional network with all shortcuts. On the other hand, setting $I_{t, i} = \{t\} \times J_i$ (where $I_{t, i}$ are sorted by lexicographic order of $(t, i)$) yields a recurrent neural network with shortcuts in depth and time. The {\em dual machine} procedure is then equivalent to a generalization of {\em backpropagation through time} in the presence of shortcuts.

\paragraph{Memory usage.} Machines' forward and backward pass computations are implemented differently from classical feedforward or recurrent neural networks. Here, we store in memory a {\em global} tensor of all units at all depths, and we update it in place in a blockwise fashion. This may appear memory-intensive compared to traditional architectures. For instance, when computing the forward pass of a feedforward neural network without shortcuts, the outputs of all but the most recently computed layer can be discarded. However, those values are needed to compute gradients by backpropagation and are stored in memory by the automatic differentiation engine. Hence, machines and neural networks have comparable memory usage during training.

\section{Conclusions}

We provide solid functional foundations for the study of deep neural networks. Borrowing ideas from functional analysis, we define the abstract notion of {\em machine}, whose {\em resolvent} generalizes the computation of a feedforward neural network. It is a unified concept that encompasses several flavors of manually designed neural network architectures, both equivariant (convolutional~\cite{lecun1995convolutional} and recurrent~\cite{werbos1988generalization} neural networks) and non-equivariant (multilayer perceptron, see~\cite{rumelhart1986learning}) architectures. This approach attempts to answer a seemingly simple question: what are the {\em defining features} of deep neural networks? More practically, how can a deep neural network be specified?

On this question, current deep learning frameworks are broadly divided in two camps. On the one hand, {\em domain-specific languages} allow users to define architectures by combining a selection of pre-existing layers. On the other hand, in the {\em differentiable programming} framework, every code is a model, provided that the automatic differentiation engine can differentiate its output with respect to its parameters.
Here, we aim to strike a balance between these opposite ends of the configurability spectrum---domain-specific languages versus differentiable programming. This is done via a principled, mathematical notion of machine: an endofunction of a normed vector space respecting a simple property. A subset of machines, machines of finite depth, are a {\em computable} generalization of deep neural networks. They are inspired by {\em nilpotent} linear operators, and indeed our main theorem concerning computability generalizes a classical result of linear algebra---the identity minus a nilpotent linear operator is invertible. The output of such a machine can be computed by iterating a simple sequence, whose behavior is remindful of {\em non-normal networks}~\cite{hennequin2012non}, where the global activity can be amplified before converging to a stable state.

We use a general procedure to define several classes of machines of finite depth. As a starting point, we juxtapose linear and nonlinear continuous endofunctions of a normed vector space. This alternation between linear and nonlinear components is one of the key ingredients of the success of deep neural networks, as it allows one to obtain complex functions as a composition of simpler ones. The notion of composition of layers in neural networks is unfortunately ill-defined, especially in the presence of shortcut connections and non-sequential architectures. In the proposed {\em machine framework}, the composition is replaced by the sum, and thus sequentiality is replaced by the weaker notion of independence. We describe independence conditions to ensure that the sum of machines is again a machine, in which case we can compute its resolvent (forward pass) explicitly. This may seem counterintuitive, as the sum is a commutative operation, whereas the composition is not. However, in our framework, we can determine the order of composition of a collection of machines via their dependency structure, and thus compute the forward pass efficiently.

Once we have established how to compute the forward pass of a machine, the backward pass is entirely analogous and can be framed as a resolvent computation. This allows us to implement a backward pass computation in a time comparable to that of the forward pass, without resorting to automatic differentiation engines, provided that we can compute the derivative of the pointwise nonlinearity, which is either explicitly available or can be obtained efficiently with scalar forward-mode differentiation. In practice, we show that not only the structure but also the runtime of the backward pass are comparable to those of the forward pass and do not incur in automatic differentiation overhead~\cite{srajer2018benchmark}. We believe that encompassing both forward and backward pass within a unified computational framework can be particularly relevant in models where not only the output of the network, but also its derivatives are used in the forward pass, as for example {\em gradient-based regularization}~\cite{155328,varga2017gradient} or {\em neural partial differential equations}~\cite{zubov2021neuralpde}.

The strategy highlighted here to define machines of finite depth often generates architectures with a large number of shortcut connections. Indeed, in the machine framework, these are more natural than purely sequential architectures. Clearly, classical, sequential architectures can be recovered by forcing a subset of parameters to equal zero, thus cancelling the shortcut connections. However, this is only one of many possible ways of regularizing a machine. Several other approaches exist: setting to zero a different subset of parameters, as in the {\em lottery ticket hypothesis}~\cite{frankle2018lottery}, penalizing large differences between adjacent parameters, or, more generally, choosing a representation of the parameter space with an associated notion of smoothness, as in kernel methods~\cite{scholkopfLearningKernelsSupport2002}. We intend to investigate the relative merits of these approaches in a future work.

\section*{Author contributions}

P.V. and M.G.B devised the project. P.V. and M.G.B developed the mathematical framework. P.V. and M.G.B. developed the software to implement the framework. P.V. wrote the original draft. M.G.B. reviewed and edited.

\bibliographystyle{abbrv}
\bibliography{SmoothMachines}

\appendix

\section{Normed vector spaces and Fr\'{e}chet derivatives}
\label{sec:normedvectorspaces}

Given normed spaces $X_1, X_2$, a function $f\colon X_1 \rightarrow X_2$ is {\em differentiable} at $x_1 \in X_1$ if it can be locally approximated by a bounded linear operator $Df\left(x_1\right)$. It is {\em continuously differentiable} if it is differentiable at all points and the derivative $Df \colon X_1 \rightarrow B(X_1, X_2)$ is continuous, where $B(X_1, X_2)$ is the space of bounded linear operators with operator norm. Whenever that is the case, we will say that $f$ is $C^1$. We will also denote the space of continuously differentiable functions as $C^1(X_1, X_2)$.

We will use $^*$ to denote both the dual of a normed space, i.e. $X^* = B(X, \R)$, and the dual of each operator. In particular, $Df\left(x_1\right)^*$, the dual of the derivative, will correspond to the operator that backpropagates cotangent vectors from the output space to the input space.

The following proposition details alternative conditions which are equivalent to requiring that a given continuously differentiable map $f$ lowers to a continuously differentiable map $\f$ between quotients.
\begin{proposition}\label{prop:quotient}
    Let $X$ be a normed vector space. Let $f \in C^1(X, X)$. Let $V, W$ be closed subspaces of $X$. The following conditions are equivalent.
    \begin{enumerate}
        \item $f$ lowers to a map $\f \in C^1(X / V, X / W)$.\label{cond:lowers}
        \item For all $x \in X,$ and  $v \in V$, $f(x + v) - f(x) \in W$.\label{cond:invariant}
        \item For all $x \in X$, $(Df(x)) V \subseteq W$.\label{cond:differential_invariant}
        \item For all $x \in X$, $Df(x)$ lowers to a map $\L(x) \in B(X / V, X / W)$.\label{cond:differential_lowers}
    \end{enumerate}
\end{proposition}
\begin{proof}
    If \cref{cond:lowers} is verified, that is to say $f$ can be lowered to a quotient map $\f \in C^1(X / V, X / W)$, then necessarily, for all $v \in V$, $f(x + v)$ and $f(x)$ correspond to the same value module $W$, hence \cref{cond:invariant} is verified.
    In \cref{cond:invariant}, we can equivalently ask that $f(x + \lambda v) - f(x) \in W$ for all $\lambda \in \R,\, v \in V$.
    Let us consider the quantity
    \begin{equation*}
        f(x + \lambda v) - f(x) = \int_0^\lambda \frac{d}{ds} f(x + sv) ds = \int_0^\lambda Df(x + sv)v ds.
    \end{equation*}
    The integrand $Df(x + sv)v$ is continuous in $s$, therefore
    \begin{equation*}
        \int_0^\lambda Df(x + sv)v d \in W \text{ for all } \lambda \in \R, \, x \in X, \, v \in V
    \end{equation*}
    if and only if
    \begin{equation*}
        Df(x + sv)v \in W \text{ for all } s \in \R, \, x \in X, \, v \in V
    \end{equation*}
    or, equivalently,
    \begin{equation*}
        (Df(x)) V \subseteq W \text{ for all } x \in X,
    \end{equation*}
    hence \cref{cond:invariant,cond:differential_invariant} are equivalent.
    By the universal property of the quotient, \cref{cond:differential_lowers} is equivalent to \cref{cond:differential_invariant}, hence \cref{cond:invariant,cond:differential_invariant,cond:differential_lowers} are equivalent.
    Whenever they are all true, we can define the lowered map $\f \in C^1(X / V, X / W)$ as
    \begin{equation*}
        \f([x]) = [f(x)],
    \end{equation*}
    which is well defined thanks to \cref{cond:invariant} and has a well defined differential given by $D\f(x) = \L(x)$ as in \cref{cond:differential_lowers}. It is straightforward to verify that $D\f \colon X/V \rightarrow B(X/V, X/W)$ is continuous. Hence, \cref{cond:invariant,cond:differential_invariant,cond:differential_lowers} imply \cref{cond:lowers}.
\end{proof}

\section{Numerical experiments}
\label{sec:benchmarks}

We ran forward and backward pass of dense, convolutional, and recurrent machines, as described in \cref{sec:implementation}. The implementation and benchmarking code is implemented in the Julia programming language~\cite{bezansonJuliaFreshApproach2017}, using Flux.jl~\cite{Flux.jl-2018} for deep learning primitives, CUDA.jl~\cite{besard2018juliagpu} for GPU support, and ChainRulesCore.jl~\cite{frames_catherine_white_2022_6375037} for efficient differentiation of pointwise activation functions. The code is available at \url{https://github.com/BeaverResearch/ParametricMachinesDemos.jl}. Simulations were run on a Intel(R) Core(TM) i7-7700HQ CPU @ 2.80GHz and on a Quadro M1200 GPU. We report the minimum times found benchmarking via the BenchmarkTools package~\cite{BenchmarkTools.jl-2016}, rounded to the fifth significant digit, as well as the backward time / forward time ratio, rounded to the third decimal place. The backward pass timings indicate the time to backpropagate cotangent vectors from machine space to input space. It is assumed that the forward pass has already been computed and that its result is available.

\begin{table}
    \begin{tabular}{llllll}
        \hline
        \textbf{machine} & \textbf{size} & \textbf{device} & \textbf{forward (\si{\milli\second})} & \textbf{backward (\si{\milli\second})} & \textbf{ratio} \\ \hline
        dense & small & CPU & 22.1 & 17.6 & 0.796 \\ 
        dense & small & GPU & 806.2 & 936.6 & 1.162 \\         
        dense & medium & CPU & 224.3 & 181.9 & 0.811 \\        
        dense & medium & GPU & 782.6 & 883.1 & 1.128 \\        
        convolution & small & CPU & 100.6 & 93.4 & 0.928 \\    
        convolution & small & GPU & 1056.7 & 1131.6 & 1.071 \\ 
        convolution & medium & CPU & 29504 & 28878 & 0.979 \\  
        convolution & medium & GPU & 2054.3 & 2252.9 & 1.097 \\ 
        recurrent & small & CPU & 365.2 & 334.8 & 0.917 \\      
        recurrent & small & GPU & 4427.4 & 4542.2 & 1.026 \\    
        recurrent & medium & CPU & 41184 & 40932 & 0.994 \\     
        recurrent & medium & GPU & 7058.7 & 7118.5 & 1.008 \\   
        \hline
    \end{tabular}
    \caption{\textbf{Timings of forward and backward passes of dense, convolutional, and recurrent machines, and backward over forward ratio.} We benchmarked on a single minibatch for a small problem size (each index set $I_i$ has dimension $2$, the minibatch contains $2$ samples) and a medium problem size (each index set $I_i$ has dimension $32$, the minibatch contains $32$ samples).}\label{tbl:benchmarks}
  \end{table}

\end{document}